\documentclass[11pt]{article}
\def\tdotoggle{1}
\usepackage{dsfont}

\usepackage{amsfonts,amssymb,amsmath,amsthm,mathtools} 
\usepackage{graphicx,microtype,latexsym,lpic,bm,xspace,booktabs,wrapfig,array}

\usepackage{thmtools,thm-restate}

\usepackage{dsfont}
\usepackage[greek,english]{babel}
\addto\extrasenglish{}
\addto\extrasenglish{}

\usepackage{bbm}

\usepackage{enumitem}  
\setlist[itemize]{noitemsep,label=$-$}
\setlist[enumerate]{noitemsep}

\usepackage[usenames,dvipsnames]{xcolor}
\definecolor{Gred}{RGB}{219, 50, 54}
\definecolor{Ggreen}{RGB}{60, 186, 84}
\definecolor{Gblue}{RGB}{72, 133, 237}
\definecolor{Gyellow}{RGB}{247, 178, 16}
\definecolor{ToCgreen}{RGB}{0, 128, 0}

\usepackage{hyperref}
\hypersetup{
	colorlinks=true,
	citecolor=Sepia,
	linkcolor=Sepia,
	filecolor=Sepia,
	urlcolor=Sepia,
}

\newtheorem{theorem}{Theorem}[section]
\newtheorem{definition}[theorem]{Definition}

\newtheorem{lem}[theorem]{Lemma}

\newtheorem{cor}[theorem]{Corollary}

\newtheorem{observation}[theorem]{Observation}

\newtheorem{remark}[theorem]{Remark}

\newcommand{\myignore}[1]{}

\newcommand{\TealBlue}[1]{\textcolor{TealBlue}{#1}}
\newcommand{\Peach}[1]{\textcolor{Peach}{#1}}
\newcommand{\Cyan}[1]{\textcolor{cyan}{#1}}
\newcommand{\Red}[1]{\textcolor{red}{#1}}
\newcommand{\Navy}[1]{\textcolor{Blue}{#1}}
\newcommand{\Blue}[1]{\textcolor{Blue}{#1}}
\newcommand{\Green}[1]{\textcolor{OliveGreen}{#1}}
\newcommand{\Black}[1]{\textcolor{black}{#1}}
\newcommand{\White}[1]{\textcolor{white}{#1}}
\newcommand{\Code}[1]{\Cyan{\texttt{#1}}}
\newcommand{\Gray}[1]{\textcolor{gray}{#1}}
\newcommand{\Magenta}[1]{\textcolor{magenta}{#1}}
\newcommand{\Maroon}[1]{\textcolor{Maroon}{#1}}

\newcommand{\tTealBlue}[1]{\ifnum\tdotoggle=1\TealBlue{#1}\else{#1}\fi}
\newcommand{\tPeach}[1]{\ifnum\tdotoggle=1\Peach{#1}\else{#1}\fi}
\newcommand{\tCyan}[1]{\ifnum\tdotoggle=1\Cyan{#1}\else{#1}\fi}
\newcommand{\tRed}[1]{\ifnum\tdotoggle=1\Red{#1}\else{#1}\fi}
\newcommand{\tNavy}[1]{\ifnum\tdotoggle=1\Navy{#1}\else{#1}\fi}
\newcommand{\tBlue}[1]{\ifnum\tdotoggle=1\Blue{#1}\else{#1}\fi}
\newcommand{\tGreen}[1]{\ifnum\tdotoggle=1\Green{#1}\else{#1}\fi}
\newcommand{\tBlack}[1]{\ifnum\tdotoggle=1\Black{#1}\else{#1}\fi}
\newcommand{\tWhite}[1]{\ifnum\tdotoggle=1\White{#1}\else{#1}\fi}
\newcommand{\tCode}[1]{\ifnum\tdotoggle=1\Code{#1}\else{#1}\fi}
\newcommand{\tGray}[1]{\ifnum\tdotoggle=1\Gray{#1}\else{#1}\fi}
\newcommand{\tMagenta}[1]{\ifnum\tdotoggle=1\Magenta{#1}\else{#1}\fi}
\newcommand{\tMaroon}[1]{\ifnum\tdotoggle=1\Maroon{#1}\else{#1}\fi}

\let\OLDthebibliography\thebibliography
\renewcommand\thebibliography[1]{
	\OLDthebibliography{#1}
	\setlength{\parskip}{1.3pt}
}

\newcommand{\inbrace}[1]{\left \{ #1 \right \}}
\newcommand{\inparen}[1]{\left ( #1 \right )}
\newcommand{\insquare}[1]{\left [ #1 \right ]}
\newcommand{\inangle}[1]{\left \langle #1 \right \rangle}

\newcommand{\inabs}[1]{\lvert #1 \rvert}
\newcommand{\norm}[1]{\lVert #1 \rVert}

\newcommand{\set}[1]{\inbrace{#1}}

\DeclareMathOperator*{\sign}{sign}
\DeclareMathOperator*{\argmin}{argmin}

\DeclareMathOperator*{\Ex}{\mathbb{E}}

\renewcommand{\Pr}{\mathbf{Pr}}

\newcommand{\eps}{\varepsilon}

\newcommand{\bbN}{{\mathbb N}}

\newcommand{\bbR}{{\mathbb R}}

\let\boldm\bm
\renewcommand{\bm}{{\boldm m}}

\newcommand{\bc}{{\boldm c}}

\newcommand{\bg}{{\boldm g}}

\newcommand{\bu}{{\boldm u}}

\newcommand{\bw}{{\boldm w}}
\newcommand{\bx}{{\boldm x}}
\newcommand{\by}{{\boldm y}}
\newcommand{\bz}{{\boldm z}}

\newcommand{\calA}{\mathcal{A}}
\newcommand{\calB}{\mathcal{B}}

\newcommand{\calD}{\mathcal{D}}

\newcommand{\calH}{\mathcal{H}}

\newcommand{\calL}{\mathcal{L}}

\newcommand{\calO}{\mathcal{O}}

\newcommand{\calU}{\mathcal{U}}

\newcommand{\calX}{\mathcal{X}}
\newcommand{\calY}{\mathcal{Y}}

\newcommand{\RERM}{{\sf RERM}}
\newcommand{\ind}{\mathbbm{1}}
\newcommand{\vc}{{\rm vc}}
\newcommand{\Risk}{{\rm R}}

\title{\bf Efficiently Learning Adversarially Robust Halfspaces with Noise}
\date{}

\author{
Omar Montasser\\
Toyota Technological Institute at Chicago\\
{\tt omar@ttic.edu}\\
\and
Surbhi Goel\\
University of Texas at Austin\\
{\tt surbhi@cs.utexas.edu}\\
\and
Ilias Diakonikolas\\
University of Wisconsin-Madison\\
{\tt ilias@cs.wisc.edu}\\
\and
Nathan Srebro\\
Toyota Technological Institute at Chicago\\
{\tt nati@ttic.edu}\\
}

\begin{document}

\maketitle

\begin{abstract}
We study the problem of learning adversarially robust halfspaces in the distribution-independent setting. In the realizable setting, we provide necessary and sufficient conditions on the adversarial perturbation sets under which halfspaces are efficiently robustly learnable. In the presence of random label noise, we give a simple computationally efficient algorithm for this problem with respect to any $\ell_p$-perturbation.
\end{abstract}

\section{Introduction}
\label{intro}

Learning predictors that are robust to adversarial examples remains a major challenge in machine learning. A line of work has shown that predictors learned by deep neural networks are {\em not} robust to adversarial examples \cite{DBLP:journals/corr/GoodfellowSS14,biggio2013evasion,DBLP:journals/corr/GoodfellowSS14}. This has led to a long line of research studying different aspects of robustness to adversarial examples.

In this paper, we consider the problem of distribution-independent learning of halfspaces that are robust to adversarial examples at test time, also referred to as robust PAC learning of halfspaces. Halfspaces are binary predictors of the form $h_{\bw}(\bx) = {\rm sign}(\inangle{\bw,\bx})$, where $\bw \in \bbR^d$.

In adversarially robust PAC learning, given an instance space $\calX$ and  label space $\calY = \set{\pm1}$, we formalize an adversary -- that we would like to be robust against -- as a map $\calU: \calX \mapsto 2^{\calX}$, where $\calU(\bx)\subseteq \calX$ represents the set of perturbations (adversarial examples) that can be chosen by the adversary at test time (i.e., we require that $\bx \in \calU(\bx)$). For an unknown distribution $\calD$ over $\calX\times \calY$, we observe $m$ i.i.d.\ samples $S\sim \calD^m$, and our goal is to learn a predictor $\hat{h}: \calX \mapsto \calY$ that achieves small robust risk, 
\begin{equation}
\Risk_{\calU}(\hat{h};\calD) \triangleq \Ex_{(\bx,y) \sim \calD}\!\left[ \sup\limits_{\bz\in\calU(\bx)} \ind[\hat{h}(\bz)\neq y] \right].
\end{equation}

The information-theoretic aspects of adversarially robust learning have been studied in recent work, see, e.g., \cite{DBLP:conf/nips/SchmidtSTTM18,NIPS2018_7307,khim2018adversarial,bubeck2019adversarial,DBLP:conf/icml/YinRB19,pmlr-v99-montasser19a}. This includes studying what learning rules should be used for robust learning and how much training data is needed to guarantee high robust accuracy. It is now known that any hypothesis class $\calH$ with finite VC dimension is robustly learnable, though sometimes improper learning is necessary and the sample complexity may be exponential in the VC dimension \cite{pmlr-v99-montasser19a}.

On the other hand, the computational aspects of adversarially robust PAC learning are less understood. In this paper, we take a first step towards studying this broad algorithmic question with a focus on the fundamental problem of learning adversarially robust halfspaces. 

A first question to ask is whether efficient PAC learning implies efficient {\em robust} PAC learning, i.e., whether there is a general reduction that solves the adversarially robust learning problem. Recent work has provided strong evidence that this is not the case. Specifically, \cite{bubeck2019adversarial} showed that there exists a learning problem  that can be learned efficiently non-robustly, but is computationally intractable to learn robustly (under plausible 
complexity-theoretic assumptions). There is also more recent evidence that suggests that this is also the case in the PAC model. \cite{awasthi2019robustness} showed that it is computationally intractable to even weakly robustly learn degree-$2$ polynomial threshold functions (PTFs) with $\ell_\infty$ perturbations in the realizable setting, while PTFs of any constant degree are known to be efficiently PAC learnable non-robustly in the realizable setting. \cite{gourdeau2019hardness} showed that there are hypothesis classes that are hard to robustly PAC learn, under the assumption that it is hard to non-robustly PAC learn.

The aforementioned discussion suggests that when studying robust PAC learning, we need to characterize which types of perturbation sets $\calU$ admit {\em computationally efficient} robust PAC learners and under which noise assumptions.  In the agnostic PAC setting, it  is known that even weak (non-robust) learning of halfspaces is computationally intractable  \cite{feldman2006new,guruswami2009hardness,DiakonikolasOSW11, daniely2016complexity}. For $\ell_2$-perturbations, where $\calU(x)=\set{z:\norm{x-z}_2\leq \gamma}$, it was recently shown that the complexity of proper learning is exponential in $1/\gamma$~\cite{diakonikolas2019nearly}. 
In this paper, we focus on the realizable case and the (more challenging) case of random label noise.

We can be more optimistic in the realizable setting. Halfspaces are efficiently PAC learnable non-robustly via Linear Programming \cite{maass1994fast}, and under the margin assumption via the Perceptron algorithm \cite{rosenblatt1958perceptron}. But what can we say about robustly PAC learning halfspaces? Given a perturbation set $\calU$ and under the assumption that there is a halfspace $h_\bw$ that robustly separates the data, can we efficiently learn a predictor with small robust risk? 

Just as empirical risk minimization (ERM) is central for non-robust PAC learning, a core component of adversarially robust learning is minimizing the {\em robust} empirical risk on a dataset $S$,
\[\hat{h}\in \RERM_{\calU}(S) \triangleq \argmin_{h\in \calH} \frac{1}{m} \sum_{i=1}^{m} \sup\limits_{\bz\in\calU(\bx)} \ind[h(\bz)\neq y].\]

In this paper, we provide necessary and sufficient conditions on perturbation sets $\calU$, under which the robust empirical risk minimization $(\RERM)$ problem is {\em efficiently} solvable in the realizable setting. We show that an efficient separation oracle for $\calU$ yields an efficient solver for $\RERM_{\calU}$, while an efficient {\em approximate} separation oracle for $\calU$ is necessary for even computing the robust loss $\sup_{\bz\in\calU(\bx)} \ind[h_{\bw}(\bz)\neq y]$ of a halfspace $h_\bw$. In addition, we relax our realizability assumption and show that under random classification noise~\cite{DBLP:journals/ml/AngluinL87}, we can efficiently robustly PAC learn halfspaces with respect to any $\ell_p$ perturbation. 

\paragraph{Main Contributions} Our main contributions can be summarized as follows:
\begin{enumerate}
    \item In the realizable setting, the class of halfspaces is 
    efficiently robustly PAC learnable with respect to $\calU$, given an efficient separation oracle for $\calU$. 
    \item To even compute the robust risk with respect to $\calU$ efficiently, an efficient {\em approximate} separation oracle for $\calU$ is {\em necessary}.
    \item In the random classification noise setting, the class of halfspaces is efficiently robustly PAC learnable with respect to any $\ell_p$ perturbation. 
\end{enumerate}

\subsection{Related Work}
Here we focus on the recent work that is most closely related to the results of this paper.
\cite{awasthi2019robustness} studied the tractability of $\RERM$ with respect to $\ell_\infty$ perturbations, obtaining efficient algorithms for halfspaces in the realizable setting, 
but showing that $\RERM$ for degree-$2$ polynomial threshold functions is computationally intractable (assuming ${\rm NP}\neq {\rm RP}$). \cite{gourdeau2019hardness} studied robust learnability of hypothesis classes defined over $\{0,1\}^n$ with respect to hamming distance, and showed that monotone conjunctions are robustly learnable when the adversary can perturb only $\calO(\log n)$ bits, but are {\em not} robustly learnable even under the uniform distribution when the adversary can flip $\omega(\log n)$ bits. 

In this work, we take a more general approach, and instead of considering specific perturbation sets, we provide methods in terms of oracle access to a separation oracle for the perturbation set $\calU$, and aim to characterize which perturbation sets $\calU$ admit tractable $\RERM$.

In the non-realizable setting, the only prior work we are aware of is by~\cite{diakonikolas2019nearly} who studied the complexity of robustly learning halfspaces in the agnostic setting under $\ell_2$ perturbations.

\section{Problem Setup}

Let $\calX=\bbR^d$ be the instance space and $\calY = \set{\pm1}$ be the label space. We consider halfspaces $\calH=\{ \bx \mapsto \sign(\inangle{\bw,\bx}): \bw\in \bbR^d \}$. 

The following definitions formalize the notion of adversarially robust PAC learning in the realizable and random classification noise settings:
\begin{definition} [Realizable Robust PAC Learning]
We say $\calH$ is robustly PAC learnable with respect to an adversary $\calU$ in the \emph{realizable} setting, if there exists a learning algorithm $\calA: (\calX\times \calY)^* \mapsto \calY^\calX$  with sample complexity $m:(0,1)\to \bbN$ such that: for any $\epsilon,\delta \in (0,1)$, for every data distribution $\calD$ over $\calX \times \calY$ where there exists a predictor $h^*\in \calH$ with zero robust risk, $\Risk_{\calU}(h^*;\calD) = 0$, with probability at least $1-\delta$ over $S \sim \calD^m$,
\[\Risk_{\calU}(\calA(S);\calD) \leq \epsilon.\]
\end{definition}

\begin{definition} [Robust PAC Learning with Random Classification Noise]
Let $h^*\in \calH$ be an unknown halfspace. Let $\calD_{\bx}$ be an arbitrary distribution over $\calX$ such that $\Risk_{\calU}(h^*;\calD_{h^*})=0$, and $\eta \leq 0 < 1/2$. A \textit{noisy example oracle}, ${\rm EX}(h^*,\calD_\bx,\eta)$ works as follows: Each time ${\rm EX}(h^*,\calD_\bx,\eta)$ is invoked, it returns a labeled example $(\bx,y)$, where $\bx\sim\calD_\bx$, $y=h^*(x)$ with probability $1-\eta$ and $y=-h^*(x)$ with probability $\eta$. Let $\calD$ be the joint distribution on $(\bx,y)$ generated by the above oracle.

We say $\calH$ is robustly PAC learnable with respect to an adversary $\calU$ in the \emph{random classification noise} model, if $\exists m(\epsilon,\delta, \eta) \in \bbN \cup \set{0}$ and a learning algorithm $\calA:(\calX\times \calY)^* \mapsto \calY^\calX$, such that for every distribution $\calD$ over $\calX \times \calY$ (generated as above by a noisy oracle), with probability at least $1-\delta$ over $S \sim \calD^m$,
\[\Risk_{\calU}(\calA(S);\calD) \leq \eta +\epsilon.\]
\end{definition}

\paragraph{Sample Complexity of Robust Learning} 
Denote by $\calL^{\calU}_\calH$ the robust loss class of $\calH$,
\begin{equation*}
    \calL^{\calU}_\calH = \left\{(x,y)\mapsto \sup\limits_{z\in\calU(x)} \ind[h(z)\neq y] : h\in \calH \right\}.
\end{equation*}

It was shown by \cite{NIPS2018_7307} that for any set $\calB\subseteq\calX$ that is nonempty, closed, convex, and origin-symmetric, and an adversary $\calU$ that is defined as $\calU(\bx)=\bx+\calB$ (e.g., \ $\ell_p$-balls), the VC dimension of the robust loss of halfspaces $\vc(\calL_{\calH}^{\calU})$ is at most the standard VC dimension $\vc(\calH)=d+1$. Based on Vapnik's ``General Learning'' \cite{vapnik:82}, this implies that we have uniform convergence of robust risk with $m=\calO(\frac{d+\log(1/\delta)}{\epsilon^2})$ samples. Formally, for any $\epsilon,\delta\in(0,1)$ and any distribution $\calD$ over $\calX\times\calY$, with probability at least $1-\delta$ over $S\sim \calD^m$, 
\begin{equation}
\label{eqn:uniform_robust}
    \forall \bw \in \bbR^d, |\Risk_{\calU}(h_\bw;\calD) - \Risk_{\calU}(h_\bw;S)|\leq \epsilon.
\end{equation}

In particular, this implies that for any adversary $\calU$ that satisfies the conditions above, $\calH$ is robustly PAC learnable w.r.t. $\calU$ by minimizing the \emph{robust} empirical risk on $S$,
\begin{equation}
\label{eqn:rerm}
    \RERM_\calU(S) = \argmin\limits_{\bw\in \bbR^d} \frac{1}{m} \sum_{i=1}^{m} \sup_{\bz_i\in\calU(\bx_i)} \ind[y_i\inangle{\bw,\bz_i} \leq 0].
\end{equation}

Thus, it remains to efficiently solve the $\RERM_{\calU}$ problem. We discuss necessary and sufficient conditions for solving $\RERM_{\calU}$ in the following section.
\section{The Realizable Setting}

In this section, we show necessary and sufficient conditions for minimizing the robust empirical risk $\RERM_\calU$ on a dataset $S=\{(\bx_1,y_1),\dots,(\bx_m,y_m)\} \in (\calX \times \calY)^m$ in the realizable setting, i.e.\ when the dataset $S$ is {\em robustly} separable with a halfspace $h_{\bw^*}$ where $\bw^*\in\bbR^d$. In Theorem~\ref{thm:sufficient}, we show that an efficient separation oracle for $\calU$ yields an efficient solver for $\RERM_{\calU}$. While in Theorem~\ref{thm:necessary}, we show that an efficient {\em approximate} separation oracle for $\calU$ is necessary for even computing the robust loss $\sup_{\bz\in\calU(\bx)} \ind[h_{\bw}(\bz)\neq y]$ of a halfspace $h_\bw$.

Note that the set of allowed perturbations $\calU$ can be non-convex, and so it might seem difficult to imagine being able to solve the $\RERM_{\calU}$ problem in full generality. But, it turns out that for halfspaces it suffices to consider only convex perturbation sets due to the following observation: 

\begin{observation}
\label{obs:convex}
Given a halfspace $\bw\in \bbR^d$ and an example $(\bx,y)\in\calX\times \calY$. If  $\forall \bz\in\calU(\bx), y\inangle{\bw,\bz}>0$, then $\forall z\in\mathsf{conv}(\calU(\bx)), y\inangle{\bw,\bz}>0$. And if $\exists z\in\calU(\bx), y\inangle{\bw,\bz}\leq0$, then $\exists z\in\mathsf{conv}(\calU(\bx)), y\inangle{\bw,\bz}\leq 0$, where $\mathsf{conv}(\calU(\bx))$ denotes the convex-hull of $\calU(\bx)$. 
\end{observation}

Observation~\ref{obs:convex} shows that for any dataset $S$ that is {\em robustly} separable w.r.t. $\calU$ with a halfspace $\bw^*$, $S$ is also robustly separable w.r.t. the convex hull $\mathsf{conv}(\calU)$ using the same halfspace $\bw^*$, where $\mathsf{conv}(\calU)(x)=\mathsf{conv}(\calU(x))$. Thus, in the remainder of this section we only consider perturbation sets $\calU$ that are convex, i.e., for each $\bx$, $\calU(\bx)$ is convex.

\begin{definition}
\label{def:sep}
Denote by $\mathsf{SEP}_{\calU}$ a separation oracle for $\calU$.
$\mathsf{SEP}_{\calU}(\bx,\bz)$ takes as input $\bx,\bz\in\calX$ and either:
\begin{itemize}
    \item[$\bullet$] asserts that $\bz\in\calU(\bx)$, or
    \item[$\bullet$] returns a separating hyperplane $\bw\in\bbR^d$ such that $\inangle{\bw,\bz'}\leq \inangle{\bw,\bz}$ for all $\bz'\in\calU(\bx)$.
\end{itemize}
\end{definition}

\begin{definition}
\label{def:eta-sep}
For any $\eta>0$, denote by $\mathsf{SEP}^{\eta}_{\calU}$ an {\em approximate} separation oracle for $\calU$. $\mathsf{SEP}^{\eta}_{\calU}$ takes as input $\bx,\bz\in\calX$ and either:
\begin{enumerate}
    \item[$\bullet$] asserts that $\bz\in \calU(\bx)^{+\eta} \stackrel{{\rm def}}{=} \set{\bz:\exists \bz' \in \calU(\bx) \text{ s.t. }\norm{\bz-\bz'}_2\leq \eta }$, or
    \item[$\bullet$] returns a separating hyperplane $\bw \in \bbR^d$ such that $\inangle{\bw,\bz'}\leq \inangle{\bw,\bz} + \eta$ for all $\bz'\in \calU(\bx)^{-\eta}\stackrel{{\rm def}}{=}\set{\bz':B(\bz',\eta)\subseteq \calU(\bx)}$.
\end{enumerate}
\end{definition}

\begin{definition}
\label{def:mem}
Denote by $\mathsf{MEM}_{\calU}$ a membership oracle for $\calU$.
$\mathsf{MEM}_{\calU}(\bx,\bz)$ takes as input $\bx,\bz\in\calX$ and either:
\begin{itemize}
    \item[$\bullet$] asserts that $\bz\in\calU(\bx)$, or
    \item[$\bullet$] asserts that $\bz\notin\calU(\bx)$.
\end{itemize}
\end{definition}

When discussing a separation or membership oracle for a fixed convex set $K$, we overload notation and write $\mathsf{SEP}_K$, $\mathsf{SEP}^\eta_K$, and $\mathsf{MEM}_K$ (in this case only one argument is required). 

\subsection{An efficient separation oracle for $\calU$ is sufficient to solve $\RERM_\calU$ efficiently}

Let $\mathsf{Soln}_{S}^{\calU} = \{\bw\in \bbR^d: \forall (\bx,y)\in S, \forall \bz \in \calU(\bx), y\inangle{\bw,\bz} > 0\}$ denote the set of valid solutions for ${\sf{RERM}_{\calU}}(S)$ (see Equation~\ref{eqn:rerm}). Note that $\mathsf{Soln}_{S}^{\calU}$ is not empty since we are considering the realizable setting. Although the treatment we present here is for homogeneous halfspaces (where a bias term is not needed), the results extend trivially to the non-homogeneous case. 

Below, we show that we can efficiently find a solution $\bw \in \mathsf{Soln}_{S}^{\calU}$ given access to a separation oracle for $\calU$, $\mathsf{SEP}_{\calU}$.

\begin{theorem}
\label{thm:sufficient}
Let $\calU$ be an arbitrary convex adversary. Given access to a separation oracle for $\calU$, $\mathsf{SEP}_{\calU}$ that runs in time $\mathsf{poly}(d,b)$. There is an algorithm  that finds $\bw \in \mathsf{Soln}_{S}^{\calU}$ in $\mathsf{poly}(m,d,b)$ time where $b$ is an upper bound on the bit complexity of the valid solutions in $\mathsf{Soln}_{S}^{\calU}$ and the examples and perturbations in $S$.
\end{theorem}

Note that the polynomial dependence on $b$ in the runtime is unavoidable even in standard non-robust ERM for halfspaces, unless we can solve linear programs in strongly polynomial time, which is currently an open problem.

Theorem~\ref{thm:sufficient} implies that for a broad family of perturbation sets $\calU$, halfspaces $\calH$ are efficiently robustly PAC learnable with respect to $\calU$ in the realizable setting, as we show in the following corollary:

\begin{cor}
\label{cor:learn-realizable}
Let $\calU:\calX\mapsto2^{\calX}$ be an adversary such that $\calU(\bx)=\bx+\calB$ where $\calB$ is nonempty, closed, convex, and origin-symmetric. Then, given access to an efficient separation oracle $\mathsf{SEP}_\calU$ that runs in time $\mathsf{poly}(d,b)$, $\calH$ is robustly PAC learnable w.r.t. $\calU$ in the realizable setting in time ${\sf poly}(d,b,1/\epsilon,\log(1/\delta))$.
\end{cor}

\begin{proof}
This follows from the uniform convergence guarantee for the robust risk of halfspaces (see Equation~\eqref{eqn:uniform_robust}) and Theorem~\ref{thm:sufficient}. 
\end{proof}

This covers many types of perturbation sets that are considered in practice. For example, $\calU$ could be perturbations of distance at most $\gamma$ w.r.t. some norm $\norm{\cdot}$, such as the $\ell_\infty$ norm considered in many applications: $\calU(\bx)=\{\bz\in \calX: \norm{\bx-\bz}_\infty \leq \gamma\}$. In addition, Theorem~\ref{thm:sufficient} also implies that we can solve the $\RERM$ problem for other natural perturbation sets such as translations and rotations in images (see,  e.g.,~\cite{pmlr-v97-engstrom19a}), and perhaps mixtures of perturbations of different types (see,  e.g.,~\cite{DBLP:journals/corr/abs-1908-08016}), as long we have access to efficient separation oracles for these sets.

\paragraph{Benefits of handling general perturbation sets $\calU$:} One important implication of Theorem~\ref{thm:sufficient} that highlights the importance of having a treatment that considers general perturbation sets (and not just $\ell_p$ perturbations for example) is the following: for any efficiently computable feature map $\varphi: \bbR^r \to \bbR^{d}$, we can efficiently solve the robust empirical risk problem over the induced halfspaces $\calH_{\varphi}=\set{\bx \mapsto \sign(\inangle{\bw, \varphi(\bx)}): \bw \in \bbR^{d}}$, as long as we have access to an efficient separation oracle for the image of the perturbations $\varphi(\calU(\bx))$. Observe that in general $\varphi(\calU(\bx))$ maybe non-convex and complicated even if $\calU(\bx)$ is convex, however Observation~\ref{obs:convex} combined with the realizability assumption imply that it suffices to have an efficient separation oracle for the convex-hull $\mathsf{conv}(\varphi(\calU(\bx)))$.

Before we proceed with the proof of Theorem~\ref{thm:sufficient}, we state the following requirements and guarantees for the Ellipsoid method which will be useful for us in the remainder of the section:

\begin{lem} [see, e.g., Theorem 2.4 in \cite{bubeck2015convex}]
\label{lem:ellipsoid}
Let $K\subseteq \bbR^d$ be a convex set, and $\mathsf{SEP}_K$ a separation oracle for $K$. Then, the Ellipsoid method using $\calO(d^2b)$ oracle queries to $\mathsf{SEP}_{K}$, will find a $\bw\in K$, or assert that $K$ is empty. Furthermore, the total runtime is $\calO(d^4b)$. 
\end{lem}

The proof of Theorem~\ref{thm:sufficient} relies on two key lemmas. First, we show that efficient robust certification yields an efficient solver for the $\RERM$ problem. Given a halfspace $\bw \in \bbR^d$ and an example $(\bx,y)\in \bbR^d\times \calY$, efficient robust certification means that there is an algorithm that can {\em efficiently} either: (a) assert that $\bw$ is robust on $\calU(\bx)$, i.e.\ $\forall \bz \in\calU(\bx), y\inangle{\bw,\bz} > 0$, or (b) return a perturbation $\bz \in \calU(\bx)$ such that $y\inangle{\bw,\bz} \leq 0$.

\begin{lem} 
\label{lem:rerm}
Let $\mathsf{CERT}_{\calU}(\bw,(\bx,y))$ be a procedure that either: (a) Asserts that $\bw$ is robust on $\calU(\bx)$, i.e., $\forall \bz \in\calU(\bx), y\inangle{\bw,\bz} > 0$, or (b) Finds a perturbation $\bz \in \calU(\bx)$ such that $y\inangle{\bw,\bz} \leq 0$. If $\mathsf{CERT}_{\calU}(\bw,(\bx,\by))$ can be solved in $\mathsf{poly}(d,b)$ time, then there is an algorithm that finds $\bw \in \mathsf{Soln}_{S}^{\calU}$ in $\mathsf{poly}(m,d,b)$ time.
\end{lem}

\begin{proof}
Observe that $\mathsf{Soln}_{S}^{\calU}$ is a convex set since 
\begin{align*}
    \bw_1,\bw_2 \in \mathsf{Soln}_{S}^{\calU} &\Rightarrow \forall (\bx,y)\in S, \forall z \in \calU(\bx), y\inangle{\bw_1,\bz} > 0 \text{ and } y\inangle{\bw_2,\bz} > 0 \\
    &\Rightarrow \forall \alpha\in[0,1], \forall (\bx,y)\in S, \forall \bz \in \calU(\bx), y\inangle{\alpha \bw_1 + (1-\alpha) \bw_2,\bz} > 0\\
    &\Rightarrow \forall \alpha\in[0,1], \alpha \bw_1 + (1-\alpha) \bw_2 \in \mathsf{Soln}_{S}^{\calU}.
\end{align*}

Our goal is to find a $\bw \in \mathsf{Soln}_{S}^{\calU}$. Let $\mathsf{CERT}_{\calU}(\bw,(\bx,y))$ be an efficient robust certifier that runs in $\mathsf{poly}(d,b)$ time. We will use $\mathsf{CERT}_{\calU}(\bw,(\bx,y))$ to implement a separation oracle for $\mathsf{Soln}_{S}^{\calU}$ denoted $\mathsf{SEP}_{\mathsf{Soln}_{S}^{\calU}}$. Given a halfspace $\bw\in\bbR^d$, we simply check if $\bw$ is robustly correct on all datapoints by running $\mathsf{CERT}_{\calU}(\bw,(\bx_i,y_i))$ on each $(\bx_i,y_i)\in S$. If there is a point $(\bx_i,y_i)\in S$ where $\bw$ is not robustly correct, then we get a perturbation $\bz_i\in \calU(\bx_i)$ where $y_i\inangle{\bw,\bz_i} \leq 0$, and we return $-y_i\bz_i$ as a separating hyperplane. Otherwise, we know that $\bw$ is robustly correct on all datapoints, and we just assert that $\bw \in \mathsf{Soln}_{S}^{\calU}$.

Once we have a separation oracle $\mathsf{SEP}_{\mathsf{Soln}_{S}^{\calU}}$, we can use the Ellipsoid method (see Lemma~\ref{lem:ellipsoid}) to solve the $\RERM_\calU(S)$ problem. More specifically, with a query complexity of $\calO(d^2b)$ to  $\mathsf{SEP}_{\mathsf{Soln}_{S}^{\calU}}$ and overall runtime of ${\sf{poly}}(m, d, b)$ (this depends on runtime of $\mathsf{CERT}_{\calU}(\bw,(\bx,y))$), the Ellipsoid method will return a $\bw \in \mathsf{Soln}_{S}^{\calU}$.
\end{proof}

Next, we show that we can do efficient robust certification when given access to an efficient separation oracle for $\calU$, $\mathsf{SEP}_{\calU}$. 

\begin{lem}
\label{lem:cert}
If we have an efficient separation oracle $\mathsf{SEP}_{\calU}$ that runs in $\mathsf{poly}(d,b)$ time. Then, we can efficiently solve $\mathsf{CERT}_{\calU}(\bw,(\bx,y))$ in $\mathsf{poly}(d,b)$ time. 
\end{lem}

\begin{proof}
Given a halfspace $\bw\in \bbR^d$ and $(\bx,y)\in \bbR^d\times\calY$, we want to either: (a) assert that $\bw$ is robust on $\calU(\bx)$, i.e.\ $\forall \bz \in\calU(\bx), y\inangle{\bw,\bz} > 0$, or (b) find a perturbation $\bz \in \calU(\bx)$ such that $y\inangle{\bw,\bz} \leq 0$. Let $M(\bw,y)=\{\bz' \in \calX: y\inangle{\bw,\bz'}\leq 0\}$ be the set of all points that $\bw$ mis-labels. Observe that by definition $M(\bw,y)$ is convex, and therefore $\calU(\bx)\cap M(\bw,y)$ is also convex. We argue that having an efficient separation oracle for $\calU(\bx)\cap M(\bw,y)$ suffices to solve our robust certification problem. Because if $\calU(\bx)\cap M(\bw,y)$ is not empty, then by definition, we can find a perturbation $\bz \in \calU(\bx)$ such that $y\inangle{\bw,\bz} \leq 0$ with a separation oracle $\mathsf{SEP}_{\calU(\bx)\cap M(\bw,y)}$ and the Ellipsoid method (see Lemma~\ref{lem:ellipsoid}). If $\calU(\bx)\cap M(\bw,y)$ is empty, then by definition, $\bw$ is robustly correct on $\calU(\bx)$, and the Ellipsoid method will terminate and assert that $\calU(\bx)\cap M(\bw,y)$ is empty.

Thus, it remains to implement $\mathsf{SEP}_{\calU(\bx)\cap M(\bw,y)}$. Given a point $\bz\in\bbR^d$, we simply ask the separation oracle for $\calU(\bx)$ by calling $\mathsf{SEP}_{\calU}(\bx, \bz)$ and the separation oracle for $M(\bw,y)$ by checking if $y\inangle{\bw,\bz}\leq 0$. If $\bz \notin \calU(\bx)$ the we get a separating hyperplane $\bc$ from $\mathsf{SEP}_{\calU}$ and we can use it separate $\bz$ from $\calU(\bx)\cap M(\bw,y)$. Similarly, if $\bz\notin M(\bw,y)$, by definition, $\inangle{y\bw, \bz} > 0$ and so we can use $y\bw$ as a separating hyperplane to separate $\bz$ from $\calU(\bx)\cap M(\bw,y)$. The overall runtime of this separation oracle is $\mathsf{poly}(d,b)$, and so we can efficiently solve $\mathsf{CERT}_{\calU}(\bw,(\bx,y))$ in $\mathsf{poly}(d,b)$ time using the Ellipsoid method (Lemma~\ref{lem:ellipsoid}).
\end{proof}

We are now ready to proceed with the proof of Theorem~\ref{thm:sufficient}. 
\begin{proof}[Proof of Theorem~\ref{thm:sufficient}]

We want to efficiently solve $\RERM_{\calU}(S)$. Given that we have a separation oracle for $\calU$, $\mathsf{SEP}_{\calU}$ that runs in $\mathsf{poly}(d,b)$. By Lemma~\ref{lem:cert}, we get an efficient robust certification procedure $\mathsf{CERT}_{\calU}(\bw,(\bx,y))$. Then, by Lemma~\ref{lem:rerm}, we get an efficient solver for $\RERM_\calU$. In particular, the runtime complexity is $\mathsf{poly}(m, d,b)$.
\end{proof}

\subsection{An efficient approximate separation oracle for $\calU$ is necessary for computing the robust loss}

Our efficient algorithm for $\RERM_{\calU}$ requires a separation oracle for $\mathcal{U}$. We now show that even efficiently computing the robust loss of a halfspace $(\bw,b_0) \in \bbR^d \times \bbR$ on an example $(\bx,y)\in \bbR^d\times \calY$ requires an efficient {\em approximate} separation oracle for $\mathcal{U}$. 

\begin{theorem}
\label{thm:necessary}
Given a halfspace $\bw \in \bbR^d$ and an example $(\bx,y)\in \bbR^d\times \calY$, let $\mathsf{EVAL}_{\calU}((\bw,b_0),(\bx,y))$ be a procedure that computes the robust loss $\sup_{\bz\in\calU(\bx)} \ind[y (\inangle{\bw,\bz} +b_0) \leq 0]$ in $\mathsf{poly}(d,b)$ time, then for any $\gamma > 0$, we can implement an efficient $\gamma$-approximate separation oracle $\mathsf{SEP}^{\gamma}_{\calU}(\bx,\bz)$ in $\mathsf{poly}(d,b,\log(1/\gamma), \log(R))$ time, where $\calU(\bx)\subseteq B(0,R)$.
\end{theorem}

\begin{proof}
Let $\gamma > 0$. We will describe how to implement a $\gamma$-approximate separation oracle for $\calU$ denoted $\mathsf{SEP}^{\gamma}_{\calU}(\bx,\bz)$. Fix the first argument to an arbitrary $\bx\in\calX$. Upon receiving a point $\bz\in\calX$ as input, the main strategy is to search for a halfspace $\bw\in \bbR^d$ that can label all of $\mathcal{U}(\bx)$ with $+1$, and label the point $\bz$ with $-1$. If $\bz\notin \mathcal{U}(\bx)$ then there is a halfspace $\bw$ that separates $\bz$ from $\calU(\bx)$ because $\calU(\bx)$ is convex, but this is impossible if $\bz\in \mathcal{U}(\bx)$. Since we are only concerned with implementing an {\em approximate} separation oracle, we will settle for a slight relaxation which is to either:
\begin{enumerate}
    \item[$\bullet$] assert that $\bz$ is $\gamma$-close to $\calU(\bx)$, i.e., $\bz \in B(\calU(\bx),\gamma)$, or
    \item[$\bullet$] return a separating hyperplane $\bw$ such that $\inangle{\bw, \bz'} \leq \inangle{\bw,\bz}$ for all $\bz'\in\calU(\bx)$. 
\end{enumerate}

Let $K=\set{(\bw,b_0): \forall \bz'\in \mathcal{U}(\bx), \inangle{\bw,\bz'} +b_0 > 0}$ denote the set of halfspaces that label all of $\calU(\bx)$ with $+1$. Since $\calU(\bx)$ is nonempty, it follows by definition that $K$ is nonempty. To evaluate membership in $K$, given a query $\bw_q, b_q$, we just make a call to $\mathsf{EVAL}_{\calU}((\bw_q,b_q),(\bx,+))$. Let ${\mathsf{MEM}}_{K}(\bw_q,b_q) = 1- \mathsf{EVAL}_{\calU}((\bw_q,b_q),(\bx,+))$. This can be efficiently computed in $\mathsf{poly}(d,b)$ time. Next, for any $\eta \in (0,0.5)$, we can get an $\eta$-approximate separation oracle for $K$ denoted $\mathsf{SEP}_{K}^{\eta}$ (see Definition~\ref{def:eta-sep}) using $\calO(d b \log{(d/\eta)})$ queries to the membership oracle ${\mathsf{MEM}}_{K}$ described above \cite{lee2018efficient}. When queried with a halfspace $\Tilde{\bw} = (\bw,b_0)$, $\mathsf{SEP}_{K}^{\eta}$ either:
\begin{enumerate}
    \item[$\bullet$] asserts that $\Tilde{\bw}\in K^{+\eta}$, or
    \item[$\bullet$] returns a separating hyperplane $\bc$ such that $\inangle{\bc,\Tilde{\bw}'}\leq \inangle{\bc,\Tilde{\bw}} + \eta$ for all halfspaces $\Tilde{\bw}'\in K^{-\eta}$.
\end{enumerate}

Observe that by definition, $K^{-\eta}\subseteq K \subseteq K^{+\eta}$. Furthermore, for any $\bw \in K^{+\eta}$, by definition, $\exists \Tilde{\bw}' \in K$ such that $\norm{\Tilde{\bw}-\Tilde{\bw}'}_2\leq \eta$. Since, for each $\bz'\in\calU(\bx)$, by definition of $K$, we have $\inangle{\Tilde{\bw}', (\bz',1)}=\inangle{\bw',\bz'}+b_0 > 0$, it follows by Cauchy-Schwarz inequality that: 
\begin{equation}
\label{eqn:proof-nec}
    \inparen{\forall \Tilde{\bw}\in K^{+\eta}}\inparen{\forall \bz'\in\calU(\bx)}: \inangle{\Tilde{\bw},(\bz',1)}=\inangle{\Tilde{\bw}-\Tilde{\bw}', \bz'} + \inangle{\Tilde{\bw}',(\bz',1)} > -\eta 2R.
\end{equation}

Let $\mathsf{SEP}_{K}^{\gamma/4R}$ be a $\frac{\gamma}{4R}$-approximate separation oracle for $K$. Observe that if the distance between $\bz$ and $\calU(\bx)$ is greater than $\gamma$, it follows that there is $(\bw,b_0)$ such that:
\[\inangle{\bw, \bz}+b_0 \leq -\gamma/2 \text{ and } \inangle{\bw,\bz'}+b_0 > 0 \inparen{\forall \bz'\in \calU(\bx)}.\]

By definition of $K$, this implies that $K\cap \set{(\bw,b_0) : \inangle{\bw,\bz}+b_0\leq -\gamma/2}$ is not empty, which implies that the intersection ${K}^{+\frac{\gamma}{4R}} \cap \set{(\bw,b_0) : \inangle{\bw,\bz}+b_0\leq -\gamma/2}$ is nonempty. We also have the contrapositive, which is, if the intersection ${K}^{+\frac{\gamma}{4R}} \cap \set{(\bw,b_0) : \inangle{\bw,\bz}+b_0\leq -\gamma/2}$ is empty, then we know that $\bz\in B(\calU(\bx),\gamma)$. To conclude the proof, we run the Ellipsoid method with the approximate separation oracle $\mathsf{SEP}_{K}^{\gamma/4R}$ to search over the restricted space $\set{(\bw,b_0) : \inangle{\bw,\bz}+b_0\leq -\gamma/2}$. Restricting the space is easily done because we will use the query point $\bz$ as the separating hyperplane. Either the Ellipsoid method will find $(\bw,b_0) \in {K}^{+\frac{\gamma}{4R}} \cap \set{(\bw,b_0) : \inangle{\bw,\bz}+b_0\leq -\gamma/2}$, in which case by Equation~\ref{eqn:proof-nec}, $(\bw,b_0)$ has the property that:
\[ \inangle{\bw, \bz} + b_0\leq -\frac{\gamma}{2} \text{ and } \inangle{\bw, \bz'} + b_0 > -\frac{\gamma}{2} \inparen{\forall \bz'\in \calU(\bx)},\]
and so we return $\bw$ as a separating hyperplane between $\bz$ and $\calU(\bx)$. If the Ellipsoid terminates without finding any such $(\bw,b_0)$, this implies that the intersection ${K}^{+\frac{\gamma}{4R}} \cap \set{(\bw,b_0) : \inangle{\bw,\bz}+b_0\leq -\gamma/2}$ is empty, and therefore, by the contrapositive above, we assert that $\bz \in B(\calU(\bx), \gamma)$.
\end{proof}
\section{Random Classification Noise}

In this section, we relax the realizability assumption to random classification noise~\cite{DBLP:journals/ml/AngluinL87}. We show that for any adversary $\calU$ that represents perturbations of bounded norm 
(i.e., $\calU(x)=x+\calB$, where $\calB=\set{\delta \in \bbR^d: {\norm{\delta}_{p}} \leq \gamma},p\in[1,\infty]$), the class of halfspaces $\calH$ is efficiently robustly PAC learnable with respect to $\calU$ in the random classification noise model.

\begin{theorem}
\label{cor:random_nois}
Let $\calU:\calX\mapsto2^{\calX}$ be an adversary such that $\calU(\bx)=\bx+\calB$ where $\calB=\set{\delta \in \bbR^d: {\norm{\delta}_{p}} \leq \gamma}$ and $p\in[1,\infty]$. Then, $\calH$ is robustly PAC learnable w.r.t $\calU$ under random classification noise in time ${\sf poly}(d, 1/\eps,1/\gamma, 1/(1-2\eta), \log(1/\delta))$.
\end{theorem}

The proof of Theorem~\ref{cor:random_nois} relies on the following key lemma. We show that the structure of the perturbations $\calB$ allows us to relate the robust loss of a halfspace $\bw\in\bbR^d$ with the $\gamma$-margin loss of $\bw$. Before we state the lemma, recall that the dual norm of $\bw$ denoted $\norm{\bw}_{*}$ is defined as $\sup{\set{\inangle{\bu, \bw}: \norm{\bu}\leq 1}}$.

\begin{lem}
\label{lem:robloss}
For any $\bw,\bx\in\bbR^d$ and any $y\in\calY$, 
\[\sup_{\delta \in \calB} \ind{\set{h_{\bw}(\bx+\delta)\neq y}} = \ind{\set{y\inangle{\frac{\bw}{\norm{\bw}_{*}},\bx} \leq \gamma}}.\] 
\end{lem}

\begin{proof}
First observe that 
\begin{align*}
\sup_{\delta\in\calB} \ind{\set{h_{\bw}(\bx+\delta)\neq y}} &= \sup_{\delta\in\calB}  \ind{\set{y\inangle{\bw,\bx+\delta}\leq0}} =\ind{\set{\inf_{\delta\in\calB}y\inangle{\bw,\bx+\delta}\leq 0}}.
\end{align*}

This holds because when $\inf_{\delta\in\calB}y\inangle{\bw,\bx+\delta} > 0$, by definition $\forall\delta\in\calB, y\inangle{\bw,\bx+\delta}>0$, which implies that $\sup_{\delta\in\calB} \ind{\set{h_{\bw}(\bx+\delta)\neq y}}=0$. For the other direction, when $\sup_{\delta\in\calB} \ind{\set{h_{\bw}(\bx+\delta)\neq y}}=1$, by definition $\exists \delta\in\calB$ such that $y\inangle{\bw,\bx+\delta}\leq 0$, which implies that $\inf_{\delta\in\calB}y\inangle{\bw,\bx+\delta}\leq 0$.
To conclude the proof, by definition of the set $\calB$ and the dual norm $\norm{\cdot}_{*}$, we have
\[\inf_{\delta\in\calB}y\inangle{\bw,\bx+\delta} =y\inangle{\bw,\bx}-\sup_{\delta\in\calB}\inangle{-y\bw,\delta} =y\inangle{\bw,\bx}-\norm{\bw}_{*}\gamma.\]
\end{proof}

\autoref{lem:robloss} implies that for any distribution $\calD$ over $\calX\times \calY$,  to solve the $\gamma$-robust learning problem
\begin{equation}
    \argmin_{w\in\bbR^d} \Ex_{(\bx,y)\sim\calD}\insquare{\sup_{\delta \in \calB} \ind{\set{h_{\bw}(\bx+\delta)\neq y}}},
\end{equation}
it suffices to solve the $\gamma$-margin learning problem
\begin{equation}
\label{eqn:margin}
    \argmin_{{\norm{\bw}_{*}}=1} \Ex_{(\bx,y)\sim\calD}\insquare{\ind{\set{y\inangle{\bw,\bx} \leq \gamma}}}.
\end{equation}

We will solve the $\gamma$-margin learning problem in Equation~\eqref{eqn:margin} in the random classification noise setting using an appropriately chosen convex surrogate loss. Our convex surrogate loss and its analysis build on a convex surrogate that appears in the appendix of \cite{diakonikolas2019distribution} for learning large $\ell_2$-margin halfspaces under random classification noise w.r.t. the 0-1 loss. We note that the idea of using a convex surrogate to (non-robustly) learn large margin halfspaces in the presence of random classification noise is implicit in a number of prior works, starting with~\cite{DBLP:conf/colt/Bylander94}.

Our robust setting is more challenging for the following reasons. First, we are not interested in only ensuring small 0-1 loss, but rather ensuring small $\gamma$-margin loss. Second, we want to be able to handle all $\ell_p$ norms, as opposed to just the $\ell_2$ norm. As a result, our analysis is somewhat delicate.

Let
\[ \phi(s) = 
    \begin{cases} 
       \lambda(1-\frac{s}{\gamma}), & s > \gamma \\
        (1- \lambda)(1-\frac{s}{\gamma}), & s \leq \gamma
   \end{cases} \;.
\]
We will show that solving the following convex optimization problem:
\begin{equation}
\label{eqn:leaky_relu}
    \argmin_{{\norm{\bw}_{*}}\leq1} G_{\lambda}^{\gamma}(\bw)\stackrel{{\rm def}}{=} \Ex_{(\bx,y)\sim\calD}\insquare{\phi(y\inangle{\bw,\bx})} \;,
\end{equation}
where $\lambda=\frac{\eps\gamma/2 + \eta}{1+\eps\gamma}$,
suffices to solve the $\gamma$-margin learning problem in Equation~\eqref{eqn:margin}. Intuitively, the idea here is that for $\lambda=0$
, the $\phi$ objective is exactly a scaled hinge loss, which gives a learning guarantee w.r.t to the $\gamma$-margin loss when there is no noise ($\eta=0$). When the noise $\eta>0$, we slightly adjust the slopes, such that even correct prediction encounters a loss. The choice of the slope is based on $\lambda$ which will depend on the noise rate $\eta$ and the $\eps$-suboptimality that is required for Equation~\eqref{eqn:margin}. 

We can solve Equation~\eqref{eqn:leaky_relu} with a standard first-order method through samples using Stochastic Mirror Descent, when the dual norm $\norm{\cdot}_{*}$ is an $\ell_q$-norm. We state the following properties of Mirror Descent we will require:

\begin{lem}[see, e.g., Theorem 6.1 in \cite{bubeck2015convex}]
\label{lem:mirrordescent}
Let $G(\bw)\stackrel{{\rm def}}{=}\Ex_{(\bx,y)\sim\calD}\insquare{\ell(\bw,(\bx,y))}$ be a convex function that is $L$-Lipschitz w.r.t. $\norm{\bw}_{q}$ where $q>1$. Then, using the potential function $\psi(\bw)=\frac{1}{2}\norm{\bw}_{q}^{2}$, a suitable step-size $\eta$, and a sequence of iterates $\bw^k$ computed by the following update:
\[\bw^{k+1}=\Pi_{B_q}^{\psi} \inparen{\nabla\psi^{-1}\inparen{\nabla\psi(\bw^{k})-\eta \bg^{k}}} \;,\]
Stochastic Mirror Descent with $\calO(L^2/(q-1)\eps^2)$ stochastic gradients $\bg$ of $G$, will find an $\eps$-suboptimal point $\hat{w}$ such that $\norm{\hat{\bw}}_{q}\leq 1$ and  $G(\hat{\bw})\leq \inf_\bw G(\bw)+\eps$.
\end{lem}

\begin{remark}
When $q=1$, we will use the entropy potential function $\psi(\bw)=\sum_{i=1}^{d}w_i\log{w_i}$. In this case, Stochastic Mirror Descent will require $\calO(\frac{L^2\log{d}}{\eps^2})$ stochastic gradients.
\end{remark}

We are now ready to state our main result for this section:

\begin{theorem}
\label{thm:rcn-robust}
Let $\calX = \set{\bx\in\bbR^d:\norm{\bx}_p\leq 1}$. Let $\calD$ be a distribution over $\calX \times \calY$ such that there exists a halfspace $\bw^*\in\bbR^d$ with $\Pr_{\bx\sim\calD_{\bx}}\insquare{\inabs{\inangle{\bw^*,\bx}}>\gamma}=1$ and $y$ is generated by $h_{\bw^*}(\bx):=\sign(\inangle{\bw^*,\bx})$ corrupted by RCN with noise rate $\eta<1/2$. An application of Stochastic Mirror Descent on $G_{\lambda}^{\gamma}(\bw)$, returns, with high probability, a halfspace $\bw$ where $\norm{\bw}_q\leq 1$ with $\gamma/2$-robust misclassification error $\Ex_{(\bx,y)\sim\calD}\insquare{\ind{\set{y\inangle{\bw,\bx}\leq \gamma/2}}}\leq \eta +\eps$ in ${\sf poly}(d, 1/\eps,1/\gamma, 1/(1-2\eta))$ time. 
\end{theorem}

With Theorem~\ref{thm:rcn-robust}, the proof of Theorem~\ref{cor:random_nois} immediately follows.

\begin{proof}[Proof of Theorem~\ref{cor:random_nois}]
This follows from Lemma~\ref{lem:robloss} and Theorem~\ref{thm:rcn-robust}.
\end{proof}

\begin{remark}
In Theorem~\ref{thm:rcn-robust}, we get a $\gamma/2$-robustness guarantee assuming $\gamma$-robust halfspace $\bw^*$ that is corrupted with random classification noise. This can be strengthened to get a guarantee of $(1-c)\gamma$-robustness for any constant $c>0$.
\end{remark}

The rest of this section is devoted to the proof of Theorem$~\ref{thm:rcn-robust}$. The high-level strategy is to show that an $\eps'$-suboptimal solution to Equation~\eqref{eqn:leaky_relu} gives us an $\eps$-suboptimal solution to Equation~\eqref{eqn:margin} (for a suitably chosen $\eps'$). In Lemma~\ref{lem:surrogate-lowerbnd}, we bound from above the $\gamma/2$-margin loss in terms of our convex surrogate objective $G^\gamma_\lambda$, and in Lemma~\ref{lem:surr-upper} we show that there are minimizers of  our convex surrogate $G^\gamma_\lambda$ such that it is sufficiently small. These are the two key lemmas that we will use to piece everything together.

For any $\bw,\bx\in\bbR^d$, consider the contribution of the objective $G_\lambda^{\gamma}$ of $\bx$, denoted by $G_{\lambda}^{\gamma}(\bw,\bx)$. This is defined as $G_{\lambda}^{\gamma}(\bw,\bx) = \Ex_{y\sim\calD_y(\bx)}\insquare{\phi(y\inangle{\bw,\bx})}= \eta \phi(-z) + (1-\eta)\phi(z)$ where $z=h_{\bw^*}(\bx)\inangle{\bw,\bx}$. In the following lemma, we provide a decomposition of $G_{\lambda}^{\gamma}(\bw,\bx)$ that will help us in proving Lemmas~\ref{lem:surrogate-lowerbnd} and \ref{lem:surr-upper}.

\begin{lem}
\label{lem:surr-exact}
For any $\bw,\bx\in\bbR^d$, let $z=h_{\bw^*}(\bx)\inangle{\bw,\bx}$. Then, we have that: 
$$G_{\lambda}^{\gamma}(\bw,\bx)=\inparen{\eta- \lambda}\inparen{\frac{z}{\gamma}} + \lambda + \eta - 2 \lambda \eta + \ind{\set{-\gamma \leq z \leq \gamma}} (1-\eta)(1-2\lambda)\inparen{1-\frac{z}{\gamma}} + \ind{\set{z < -\gamma}}(1-2\lambda)\inparen{1-2\eta-\frac{z}{\gamma}} \;.$$
\end{lem}

\begin{proof}
Based on the definition of the surrogate loss, it suffices to consider three cases:
\begin{description}
    \item[Case $z > \gamma$]:
    \begin{align*}
        l_1(z) &\stackrel{{\rm def}}{=} \eta \phi(-z) + (1-\eta)\phi(z) \\
        &= \eta (1- \lambda) (1+z/\gamma) + (1-\eta)\lambda(1-z/\gamma)\\
        &= \inparen{\eta(1- \lambda) -(1-\eta)\lambda}\inparen{\frac{z}{\gamma}} + (1-\eta)\lambda + \eta(1- \lambda)\\
        &= \inparen{\eta- \lambda}\inparen{\frac{z}{\gamma}} + \lambda + \eta - 2 \lambda \eta.
    \end{align*}
    \item[Case $-\gamma \leq z \leq \gamma$]:
    \begin{align*}
        l_2(z) &\stackrel{{\rm def}}{=} \eta \phi(-z) + (1-\eta)\phi(z) \\
        &= \eta(1- \lambda)(1+z/\gamma)+(1-\eta)(1- \lambda)(1-z/\gamma)\\
        &=-(1-2\eta)(1 - \lambda) \inparen{\frac{z}{\gamma}} + 1 - \lambda.
    \end{align*}
    \item[Case $z < -\gamma$]:
    \begin{align*}
        l_3(z) &\stackrel{{\rm def}}{=} \eta\lambda(1+z/\gamma) + (1-\eta)(1 - \lambda)(1-z/\gamma) \\
    &= (\eta\lambda - (1-\eta)(1- \lambda))\inparen{ \frac{z}{\gamma}} + \eta\lambda + (1-\eta)(1- \lambda)\\
    &= (\eta + \lambda - 1)\inparen{ \frac{z}{\gamma}} + 1 -\eta - \lambda + 2\eta\lambda.
    \end{align*}
\end{description}
Considering the three cases above, we can write
\begin{equation*}
G(z)= l_1(z) + \ind{\set{-\gamma \leq z \leq \gamma}} \inparen{l_2(z) - l_1(z)} + \ind{\set{z < -\gamma}}\inparen{l_3(z) - l_1(z)} \;.
\end{equation*}
Then, we calculate $l_2(z)-l_1(z)$ and $l_3(z)-l_1(z)$ as follows:
\begin{align*}
    l_2(z) - l_1(z) &= -(1-2\eta)(1 - \lambda) \inparen{\frac{z}{\gamma}} + 1 - \lambda - \inparen{\eta- \lambda}\inparen{\frac{z}{\gamma}} - \lambda - \eta + 2 \lambda \eta\\
    &= -((1-2\eta)(1 - \lambda) + \eta - \lambda) \inparen{\frac{z}{\gamma}} + (1 - \eta)(1 - 2\lambda)\\
    &= (1 - \eta)(1 - 2\lambda) \inparen{1 - \frac{z}{\gamma}} \;, \textrm{ and }\\
    l_3(z)-l_1(z) &= (\eta + \lambda - 1)\inparen{ \frac{z}{\gamma}} + 1 -\eta - \lambda + 2\eta\lambda - \inparen{\eta- \lambda}\inparen{\frac{z}{\gamma}} - \lambda - \eta + 2 \lambda \eta\\
    &= (2\lambda - 1)\inparen{\frac{z}{\gamma}} + (1 -2\eta) (1 - 2 \lambda)\\
    &= (2\lambda - 1)\inparen{\frac{z}{\gamma} + 2\eta - 1}.
\end{align*}
Using the above, we have that 
\begin{align*}
    G(z) &= l_1(z) + \ind{\set{-\gamma \leq z \leq \gamma}} (1-\eta)(1-2\lambda)\inparen{1-\frac{z}{\gamma}} \\
&+ \ind{\set{z < -\gamma}}(1-2\lambda)\inparen{1-2\eta-\frac{z}{\gamma}}.
\end{align*}
\end{proof}

The following lemma allows us to bound from below our convex surrogate $\Ex_{\bx\sim\calD_\bx}[G_{\lambda}^{\gamma}(\bw,\bx)]$ in terms of the $\gamma/2$-margin loss of $\bw$. 

\begin{lem}
\label{lem:surrogate-lowerbnd}
Assume that $\lambda$ is chosen such that $\lambda < 1/2$ and $\eta < \lambda$. Then, for any $\bw\in\bbR^d$, $\Ex_{\bx\sim\calD_\bx}[G_{\lambda}^{\gamma}(\bw,\bx)] \geq 
\frac{\eta-\lambda}{\gamma} + \frac{1}{2}(1-2\lambda)(1-\eta)\Ex_{x}\insquare{\ind{\set{z \leq \frac{\gamma}{2}}}} + \lambda + \eta - 2 \lambda \eta$.
\end{lem}

\begin{proof}
By \autoref{lem:surr-exact}, and linearity of expectation, we have that 
\begin{eqnarray*}
    \Ex_{\bx\sim\calD_\bx}[G_{\lambda}^{\gamma}(\bw,\bx)]
    &=& \inparen{\eta- \lambda}\Ex_{\bx\sim\calD_\bx}\insquare{\frac{z}{\gamma}} + \lambda + \eta - 2 \lambda \eta + (1-\eta)(1-2\lambda) \Ex_{\bx\sim\calD_\bx}\insquare{\ind{\set{-\gamma \leq z \leq \gamma}}\inparen{1-\frac{z}{\gamma}}} \\
    &+& (1-2\lambda)\Ex_{\bx\sim\calD_\bx}\insquare{\ind{\set{z < -\gamma}}\inparen{1-2\eta-\frac{z}{\gamma}}} \;.
\end{eqnarray*}
First, observe that for any $\bx$, $ -1 \leq z=h_{\bw^*}(\bx)\inangle{\bw,\bx} \leq 1$ and since $\eta\le  \lambda$, we have 
\[(\eta- \lambda)\Ex_{\bx\sim\calD_\bx}\insquare{\frac{z}{\gamma}} \geq \frac{\eta -\lambda}{\gamma} \;.\]
Then we observe that whenever $z<-\gamma$, $1-\frac{z}{\gamma}-2\eta > 2(1-\eta) > (1- \eta)/2$ and $\lambda \le 1/2$, thus we can bound from below the third term 
\begin{equation*}
(1-2\lambda)\Ex_{\bx\sim\calD_\bx}\insquare{\ind{\set{z < -\gamma}}\inparen{1-2\eta-\frac{z}{\gamma}}} \geq \frac{1}{2}(1 - 2\lambda)(1-\eta)\Ex_{\bx\sim\calD_\bx}\insquare{\ind{\set{z < -\gamma}}} \;.
\end{equation*}
Next we note that whenever $-\gamma \leq z \leq \gamma$, $1-\frac{z}{\gamma}\geq0$. This implies that instead of considering $\ind{\set{-\gamma \leq z \leq \gamma}}$, we can relax this and consider the subset $\ind{\set{-\gamma \leq z \leq \frac{\gamma}{2}}}$, and on this subset $1-\frac{z}{\gamma}\geq 1/2$. Thus, we can bound the second term from below as follows:
\begin{equation*}
        (1 - 2 \lambda)(1-\eta)\Ex_{\bx\sim\calD_\bx}\insquare{\ind{\set{-\gamma \leq z \leq \gamma}}\inparen{1-\frac{z}{\gamma}}} \geq \frac{1}{2}(1-2\lambda)(1-\eta)\Ex_{\bx\sim\calD_\bx}\insquare{\ind{\set{-\gamma \leq z \leq \frac{\gamma}{2}}}}.
\end{equation*}
Combining the above, we obtain
\begin{eqnarray*}
    \Ex_{\bx\sim\calD_\bx}[G_{\lambda}^{\gamma}(\bw,\bx)] 
    &\geq& \frac{\eta- \lambda}{\gamma}+\frac{1}{2}(1-2\lambda)(1-\eta)\Ex_{x}\insquare{\ind{\set{-\gamma \leq z \leq \frac{\gamma}{2}}}+\ind{\set{z < -\gamma}}}+ \lambda + \eta - 2 \lambda \eta\\
    &\geq& \frac{\eta-\lambda}{\gamma} + \frac{1}{2}(1-2\lambda)(1-\eta)\Ex_{x}\insquare{\ind{\set{z \leq \frac{\gamma}{2}}}} + \lambda + \eta - 2 \lambda \eta \;,
\end{eqnarray*}
as desired.
\end{proof}

We now show that there exist minimizers of the convex surrogate $G^\gamma_\lambda$ such that it is sufficiently small, which will be useful later in choosing the suboptimality parameter $\epsilon'$. 

\begin{lem}
\label{lem:surr-upper}
Assume that $\lambda$ is chosen such that $\lambda < 1/2$ and $\eta < \lambda$. Then we have that
\[\inf_{\bw\in\bbR^d} \Ex_{\bx\sim\calD_\bx}[G_{\lambda}^{\gamma}(\bw,\bx)] \le 2\eta(1 - \lambda).\]
\end{lem}

\begin{proof}
By definition, we have that $\inf_{\bw\in\bbR^d} \Ex_{\bx\sim\calD_\bx}[G_{\lambda}^{\gamma}(\bw,\bx)]\leq \Ex_{\bx\sim\calD_\bx}[G_{\lambda}^{\gamma}(\bw^*,\bx)]$. By assumption, with probability 1 over $\bx\sim\calD_{\bx}$, we have $h_{\bw^*}(\bx)\inangle{\bw^*,\bx} > \gamma$. Thus, by \autoref{lem:surr-exact}, we have 
\begin{align*}
    \Ex_{\bx}[G_{\lambda}^{\gamma}(\bw^*,\bx)] = \inparen{\eta- \lambda}\inparen{\frac{z}{\gamma}} + \lambda + \eta - 2 \lambda \eta
    \le \eta - \lambda + \lambda + \eta - 2 \lambda \eta = 2\eta(1 - \lambda) \;,
\end{align*}
where the last inequality follows from the fact that $\eta < \lambda$ and the fact that $\Ex_{\bx}\insquare{\frac{z}{\gamma}}>1$.
\end{proof}
Using the above lemmas, we are now able to bound from above the $\gamma/2$-margin loss of a halfspace $\bw$ that is $\eps'$-suboptimal for our convex optimization problem (see Equation~\eqref{eqn:leaky_relu}).
\begin{lem}
\label{lem:subopt}
For any $\eps' \in (0,1)$ and any $\bw \in \bbR^d$ such that $\Ex_{x\sim\calD_x}[G_{\lambda}^{\gamma}(\bw,\bx)]\leq \Ex_{x\sim\calD_x}[G_{\lambda}^{\gamma}(\bw^*,\bx)] + \eps'$, the $\gamma/2$-missclassification error of $\bw$
satisfies
$$\Ex_{(\bx,y)\sim\calD}\insquare{\ind{\set{y\inangle{\bw,\bx}\leq \gamma/2}}} \leq \eta + \frac{2}{(1-2\lambda)}\inparen{\eps' + (\lambda - \eta)\inparen{\frac{1}{\gamma} - 1}} \;.$$ 
\end{lem}

\begin{proof}
By \autoref{lem:surrogate-lowerbnd} and \autoref{lem:surr-upper}, we have 
\begin{align*}
       \frac{\eta-\lambda}{\gamma} + \frac{1}{2}(1-2\lambda)(1-\eta)\Ex_{x}\insquare{\ind{\set{z \leq \frac{\gamma}{2}}}} + \lambda + \eta - 2 \lambda \eta \leq 2\eta(1 - \lambda) + \eps'.
\end{align*}
This implies
\[
(1- \eta)\Ex_{x}\insquare{\ind{\set{z \leq \frac{\gamma}{2}}}} \le  \frac{2}{(1-2\lambda)}\inparen{\eps' + (\lambda - \eta)\inparen{\frac{1}{\gamma} - 1}}.
\]
Since $\Ex_{(\bx,y)\sim\calD}\insquare{\ind{\set{y\inangle{\bw,\bx}\leq \gamma/2}}} \le \eta + (1- \eta)\Ex_{x}\insquare{\ind{\set{z \leq \frac{\gamma}{2}}}}$, we get the desired result.
\end{proof}

We are now ready to prove 
Theorem~\ref{thm:rcn-robust}.

\begin{proof}[ Proof of Theorem~\ref{thm:rcn-robust}]
Based on \autoref{lem:subopt}, we will choose $\lambda, \eps'$ such that $$\frac{2}{(1-2\lambda)}\inparen{\eps' + (\lambda - \eta)\inparen{\frac{1}{\gamma} - 1}} \leq \eps \;.$$ 
By setting $\eps' = \lambda - \eta$, this condition reduces to
\[\frac{2(\lambda - \eta)}{(1-2\lambda)} \leq \gamma\eps \;.\]
This implies that we need $\lambda \leq \frac{\eps\gamma/2 + \eta}{1+\eps\gamma}$. We will choose $\lambda=\frac{\eps\gamma/2 + \eta}{1+\eps\gamma}$. Note that our analysis relied on having $\lambda \le 1/2$ and $\eta\le \lambda$. These conditions combined imply that we should choose $\lambda$ such that $\eta \le \lambda \le 1/2$. Our choice of $\lambda=\frac{\eps\gamma/2 + \eta}{1+\eps\gamma}$ satisfies these conditions, since
\[ \frac{\eps\gamma/2 + \eta}{1+\eps\gamma} - \eta = \frac{\eps\gamma/2 + \eta - \eta (1+\eps\gamma)}{1+\eps\gamma} = \frac{\eps\gamma(1/2-\eta)}{1+\eps\gamma} \ge 0 \;,\]
and 
\[ \frac{\eps\gamma/2 + \eta}{1+\eps\gamma} - \frac{1}{2} =  \frac{\eps\gamma/2 + \eta - 1/2(1+\eps\gamma)}{1+\eps\gamma}=\frac{\eta - 1/2}{1+\eps\gamma} \le 0 \;.\]
By our choice of $\lambda$, we have that $\eps' = \lambda -\eta = \frac{\eps\gamma(1/2-\eta)}{1+\eps\gamma}=\frac{\eps\gamma(1-2\eta)}{2(1+\eps\gamma)}$. By the guarantees of Stochastic Mirror Descent (see \autoref{lem:mirrordescent}), our theorem follows with $\calO\inparen{\frac{1}{\eps^2\gamma^2(1-2\eta)^2(q-1)}}$ samples for $q>1$ and  $\calO\inparen{\frac{\log d}{\eps^2\gamma^2(1-2\eta)^2}}$ samples for $q=1$. 
\end{proof}

\section{Conclusion}

In this paper, we provide necessary and sufficient conditions for perturbation sets $\calU$, under which we can efficiently solve the robust empirical risk minimization $(\RERM)$ problem. We give a polynomial time algorithm to solve $\RERM$ given access to a polynomial time separation oracle for $\calU$. In addition, we show that an efficient {\em approximate} separation oracle for $\calU$ is necessary for even computing the robust loss of a halfspace. As a corollary, we show that halfspaces are efficiently robustly PAC learnable for a broad range of perturbation sets. By relaxing the realizability assumption, we show that under random classification noise, we can efficiently robustly PAC learn halfspaces with respect to any $\ell_p$ perturbations. An interesting direction for future work 
is to understand the computational complexity of robustly PAC learning halfspaces under stronger noise models, including Massart noise and agnostic noise.
\section*{Acknowledgments}

We thank Sepideh Mahabadi for many insightful and helpful discussions, and Brian Bullins for helpful discussions on Mirror Descent. This work was initiated when the authors were visiting the Simons Institute for the Theory of Computing as part of the Summer 2019 program on {\it Foundations of Deep Learning}. Work by N.S. and O.M. was partially supported by NSF award IIS-1546500 and by DARPA\footnote{This paper does not reflect the position or the policy of the Government, and no endorsement should be inferred.} cooperative agreement HR00112020003. I.D. was supported by NSF Award CCF-1652862 (CAREER), a Sloan Research Fellowship, and a DARPA Learning with Less Labels (LwLL) grant. S.G. was supported by the JP Morgan AI Research PhD Fellowship.


\addcontentsline{toc}{section}{References}
\DeclareUrlCommand{\Doi}{\urlstyle{sf}}
\renewcommand{\path}[1]{\small\Doi{#1}}
\renewcommand{\url}[1]{\href{#1}{\small\Doi{#1}}}
\bibliographystyle{alphaurl}
\bibliography{refs}

\newpage
\onecolumn

\section*{Appendix: Another Approach to Large Margin Learning under Random Classification Noise}

We remark that $\gamma$-margin learning of halfspaces has been studied in earlier work, and we have algorithms such as Margin Perceptron \cite{DBLP:journals/ml/BalcanBS08} and SVM. The Margin Perceptron (for $\ell_2$ margin) and other $\ell_p$ margin algorithms have been also implemented in the SQ model \cite{DBLP:conf/soda/FeldmanGV17}. But no explicit connection has been made to adversarial robustness.

We present here a simple approach to learn $\gamma$-margin halfspaces under random classification noise using only a convex surrogate loss and Stochastic Mirror Descent. The construction of the convex surrogate is based on learning generalized linear models with a suitable link function $u:\bbR \to \bbR$. To the best of our knowledge, the result of this section is not explicit in prior work.

\begin{theorem}
\label{thm:rcn-robust-2}
Let $\calX = \set{\bx\in\bbR^d:\norm{\bx}_p\leq 1}$. Let $\calD$ be a distribution over $\calX \times \calY$ such that there exists a halfspace $\bw^*\in\bbR^d,\norm{\bw^*}_q=1$ with $\Pr_{\bx\sim\calD_{\bx}}\insquare{\inabs{\inangle{\bw^*,\bx}}>\gamma}=1$ and $y$ is generated by $h_{\bw^*}(\bx):=\sign(\inangle{\bw^*,\bx})$ corrupted with random classification noise rate $\eta<1/2$. Then, running Stochastic Mirror Descent on the following convex optimization problem:
\[\min_{\bw\in\bbR^d,\norm{\bw}_{q}\leq 1} \Ex_{(\bx,y)\sim\calD} \insquare{\ell(\bw,(\bx,y))}\]
where the convex loss function $\ell$ is defined in Equation~\ref{eqn:gsurr}, returns with high probability, a halfspace $\bw$ with $\gamma/2$-robust misclassification error $\Ex_{(\bx,y)\sim\calD}\insquare{\ind{\set{y\inangle{\bw,\bx}\leq \gamma/2}}}\leq \eta +\eps$. 
\end{theorem}

We prove \autoref{thm:rcn-robust} in the remainder of this section. We will connect our problem to that of solving generalized linear models. We define the link function as follows, 
\[ u(s) = 
    \begin{cases} 
       \eta& s < -\gamma \\
        \frac{1-2\eta}{2\gamma} s + \frac{1}{2}    & -\gamma \leq s \leq \gamma\\
        1-\eta& s > \gamma
   \end{cases}.
\]
Observe that $u$ is monotone and $\frac{1-2\eta}{2\gamma}$-Lipschitz. 

First, we will relate our loss of interest, which is the $\gamma/2$-margin loss with the squared loss defined in terms of the link function $u$, 
\begin{lem}
\label{lem:link}
For any $\bw\in\bbR^d$, 
\begin{equation*}
 \Ex_{\bx\sim\calD_{\bx}} \insquare{\ind{\set{h_{\bw^*}(\bx)\inangle{\bw,\bx}\leq \gamma/2}}} \leq  \frac{16}{(1-2\eta)^2} \Ex_{\bx\sim\calD_{\bx}}\insquare{\inparen{u(\inangle{\bw,\bx}) - u(\inangle{\bw^*,\bx})}^2}.
\end{equation*}
\end{lem}

\begin{proof}
Let $E^+=\set{h_{\bw^*}(\bx)=+}$ and $E^-=\set{h_{\bw^*}(\bx)=-}$. By law of total expectation and definition of the link function $u$, we have
{\small
\begin{align*}
    \Ex_{\bx\sim\calD_{\bx}} \insquare{\inparen{u(\inangle{\bw,\bx}) - u(\inangle{\bw^*,\bx})}^2} &= \Ex_{\bx\sim\calD_{\bx}} \insquare{\inparen{u(\inangle{\bw,\bx}) - u(\inangle{\bw^*,\bx})}^2\ind{\set{E^+}}} + \Ex_{\bx\sim\calD_{\bx}} \insquare{\inparen{u(\inangle{\bw,\bx}) - u(\inangle{\bw^*,\bx})}^2\ind{\set{E^-}}}\\
    &= \Ex_{\bx\sim\calD_{\bx}} \insquare{\inparen{u(\inangle{\bw,\bx}) - (1-\eta)}^2\ind{\set{E^+}}} + \Ex_{\bx\sim\calD_{\bx}} \insquare{\inparen{u(\inangle{\bw,\bx}) - \eta}^2\ind{\set{E^-}}}.
\end{align*}
}
We will lower bound both terms:
\begin{align*}
    \Ex_{\bx\sim\calD_{\bx}} \insquare{\inparen{u(\inangle{\bw,\bx}) - (1-\eta)}^2\ind{\set{E^+}}} &\geq \Ex_{\bx\sim\calD_{\bx}} \insquare{\inparen{a - (1-\eta)}^2\ind{\set{E^+}}\ind{\set{u(\inangle{\bw,\bx})\leq a}}}\\
    \Ex_{\bx\sim\calD_{\bx}} \insquare{\inparen{u(\inangle{\bw,\bx}) - \eta}^2\ind{\set{E^-}}} &\geq \Ex_{\bx\sim\calD_{\bx}} \insquare{\inparen{b - \eta}^2\ind{\set{E^-}}\ind{\set{u(\inangle{\bw,\bx})\geq b}}}
\end{align*}

Then, observe that the event $\set{\inangle{\bw,\bx}\leq \gamma/2}$ implies the event $\set{u(\inangle{\bw,\bx})\leq \frac{3-2\eta}{4}}$, and similarly the event $\set{\inangle{\bw,\bx}\geq -\gamma/2}$ implies the event $\set{u(\inangle{\bw,\bx})\geq \frac{2\eta+1}{4}}$. This means that 
{\footnotesize
\begin{align*}
    \Ex_{\bx\sim\calD_{\bx}} \insquare{\inparen{\frac{3-2\eta}{4} - (1-\eta)}^2\ind{\set{E^+}}\ind{\set{u(\inangle{\bw,\bx})\leq \frac{3-2\eta}{4}}}} &\geq \inparen{\frac{3-2\eta}{4} - (1-\eta)}^2\Ex_{\bx\sim\calD_{\bx}}\insquare{\ind{\set{E^+}}\ind{\set{\inangle{\bw,\bx}\leq \gamma/2}}}\\
    \Ex_{\bx\sim\calD_{\bx}} \insquare{\inparen{\frac{2\eta+1}{4} - \eta}^2\ind{\set{E^-}}\ind{\set{u(\inangle{\bw,\bx})\geq \frac{2\eta+1}{4}}}} &\geq \inparen{\frac{2\eta+1}{4} - \eta}^2 \Ex_{\bx\sim\calD_{\bx}} \insquare{\ind{\set{E^-}}\ind{\set{\inangle{\bw,\bx}\geq -\gamma/2}}}
\end{align*}
}
We combine these observations to conclude the proof,
\begin{align*}
    \Ex_{\bx\sim\calD_{\bx}}\insquare{\inparen{u(\inangle{\bw,\bx}) - u(\inangle{\bw^*,\bx})}^2} &\geq \frac{(1-2\eta)^2}{16}\Ex_{\bx\sim\calD_{\bx}}\insquare{\ind{\set{E^+}}\ind{\set{\inangle{\bw,\bx}\leq \gamma/2}}+\ind{\set{E^-}}\ind{\set{\inangle{\bw,\bx}\geq -\gamma/2}}}\\
    &\geq \frac{(1-2\eta)^2}{16}\Ex_{\bx\sim\calD_{\bx}} \insquare{\ind{\set{h_{\bw^*}(\bx)\inangle{\bw,\bx}\leq \gamma/2}}}.
\end{align*}
\end{proof}

But note that the squared loss is non-convex and so it may not be easy to optimize. Luckily, we can get a tight upper-bound with the following surrogate loss (see \cite{varun}):
\begin{equation}
\label{eqn:gsurr}
    \ell(\bw,(\bx,y)) = \int_{0}^{\inangle{\bw,\bx}} (u(s) - y)ds.
\end{equation}
Note that $\ell(\bw,(\bx,y))$ is convex w.r.t $\bw$ since the Hessian $\nabla_{\bw}^2\ell(\bw, (\bx,y)) = u'(\inangle{\bw,\bx})\bx\bx^T$ is positive semi-definite. 

Assuming our labels $y$ have been transformed to $\{0,1\}$ from $\{\pm 1\}$, observe that $\Ex[y|x] = u(\inangle{\bw^*,\bx})$. We now have the following guarantee (see e.g., \cite{cohen2014surrogate,varun}):
\begin{equation}
\label{eqn:surrogate}
\Ex_{\bx\sim\calD_{\bx}} \insquare{\inparen{u(\inangle{\bw,\bx}) - u(\inangle{\bw^*,\bx})}^2} \leq 2 \frac{1-2\eta}{2\gamma} \Ex_{(\bx,y)\sim\calD} \insquare{\ell(\bw,(\bx,y)) - \ell(\bw^*,(\bx,y))}.
\end{equation}

\begin{proof}[Proof of Theorem \ref{thm:rcn-robust-2}]
Combining \autoref{lem:link} and \autoref{eqn:surrogate}, we get the following guarantee for any $\bw\in\bbR^d$,
\[(1-2\eta)\Ex_{\bx\sim\calD_{\bx}} \insquare{\ind{\set{h_{\bw^*}(\bx)\inangle{\bw,\bx}\leq \gamma/2}}} \leq \frac{16}{\gamma}\Ex_{(\bx,y)\sim\calD}\insquare{\ell(\bw,(\bx,y)) - \ell(\bw^*,(\bx,y))}.\]

Thus, running Stochastic Mirror Descent with $\eps'=(\eps\gamma(1-2\eta))/16$ and $O(1/\eps'^2)$ samples (labels transformed to $\{0,1\}$), returns with high probability, a halfspace $\bw$ such that 
\begin{equation}
\label{eqn:10-margin-upper}
    \Ex_{\bx\sim\calD_{\bx}} \insquare{\ind{\set{h_{\bw^*}(\bx)\inangle{\bw,\bx}\leq \gamma/2}}}\leq \eps.
\end{equation}

Then, to conclude the proof, observe that
\begin{align*}
    \Ex_{(\bx,y)\sim\calD}\insquare{\ind{\set{y\inangle{\bw,\bx}\leq \gamma/2}}} &= \eta \Ex_{\bx\sim\calD_\bx}\insquare{\ind{\set{-h_{\bw^*}(\bx)\inangle{\bw,\bx}\leq \gamma/2}}} + (1-\eta)\Ex_{\bx\sim\calD_\bx}\insquare{\ind{\set{h_{\bw^*}(\bx)\inangle{\bw,\bx}\leq \gamma/2}}}\\
    &= \eta (1-\Ex_{\bx\sim\calD_\bx}\insquare{\ind{\set{h_{\bw^*}(\bx)\inangle{\bw,\bx}\le -\gamma/2}}}) +(1-\eta) \Ex_{\bx\sim\calD_\bx}\insquare{\ind{\set{h_{\bw^*}(\bx)\inangle{\bw,\bx}\leq \gamma/2}}}\\
    &\stackrel{(i)}{\leq} \eta + (1-\eta)\Ex_{\bx\sim\calD_\bx}\insquare{\ind{\set{h_{\bw^*}(\bx)\inangle{\bw,\bx}\leq \gamma/2}}}\\
    &\stackrel{(ii)}{\leq} \eta + \eps,
\end{align*}
where $(i)$ follows from that fact that $\Ex_{\bx\sim\calD_\bx}\insquare{\ind{\set{h_{\bw^*}(\bx)\inangle{\bw,\bx}\le -\gamma/2}}}\geq 0$, and $(ii)$ follows from Equation~\ref{eqn:10-margin-upper}. 
\end{proof}
\end{document}